\documentclass[11pt,a4paper]{article}
\usepackage[]{authblk}
\usepackage[hyperref]{acl2021}
\usepackage{times}
\usepackage{latexsym}

\usepackage{microtype}

\aclfinalcopy % Uncomment this line for the final submission
\def\aclpaperid{3079} % Enter the acl Paper ID here
\usepackage{amsfonts}
\usepackage{amsmath}
\usepackage{amsthm}
\usepackage[lined,ruled]{algorithm2e}
\usepackage{algpseudocode}
\usepackage{graphicx}\usepackage{multirow}
\usepackage{enumitem}
\usepackage{color}
\usepackage{tikz}
\usepackage{xspace}
\usepackage{subcaption}
\usepackage[percent]{overpic}

\newcommand{\baseDP}{\textsc{SanText}\xspace}
\newcommand{\baseDPp}{\textsc{SanText}$^+$\xspace}

\newcommand{\eg}{{\it e.g.}\xspace}

\newcommand{\ie}{{\it i.e.}\xspace}

\newcommand{\M}{\mathcal{M}}
\newcommand{\X}{\mathcal{X}}
\newcommand{\Y}{\mathcal{Y}}
\newcommand{\V}{\mathcal{V}}
\renewcommand{\S}{\mathcal{S}}

\newcommand{\euc}{\mathsf{euc}}
\newcommand{\privacy}{\mathsf{privacy}}

\newtheorem{definition}{\bf Definition}
\newtheorem{theorem}{\bf Theorem}

\definecolor{earthyellow}{rgb}{0.88, 0.66, 0.37}

\newcommand{\nop}[1]{}

\makeatletter
\def\@fnsymbol#1{\ensuremath{\ifcase#1\or \dagger\or \ddagger\or
   \mathsection\or \mathparagraph\or \|\or **\or \dagger\dagger
   \or \ddagger\ddagger \else\@ctrerr\fi}}
    \makeatother

\title{Differential Privacy for Text Analytics via Natural Text Sanitization\thanks{\ Our code is available at \url{https://github.com/xiangyue9607/SanText}.
}}

\author[1,\thanks{\ \ The first two authors contributed equally.}]{Xiang Yue}
\author[2]{Minxin Du}
\author[3]{Tianhao Wang}
\author[4]{Yaliang Li}
\author[1]{Huan Sun}
\author[2]{Sherman S. M. Chow}

\affil[1]{The Ohio State University}
\affil[2]{The Chinese University of Hong Kong}
\affil[3]{Carnegie Mellon University}
\affil[4]{Alibaba Group}
\affil[ ]{\{\texttt{yue.149,sun.397\}@osu.edu,\{\texttt{dm018,sherman\}@ie.cuhk.edu.hk}}}
\affil[ ]{\texttt{tianhao@cmu.edu},\quad \texttt{yaliang.li@alibaba-inc.com}}

\begin{document}
\maketitle

\begin{abstract}
Texts convey sophisticated knowledge.
However, texts also convey sensitive information. 
Despite the success of general-purpose language models and domain-specific mechanisms with differential privacy (DP), existing text sanitization mechanisms still provide low utility, as cursed by the high-dimensional text representation. 
The companion issue of utilizing sanitized texts for downstream analytics is also under-explored. 
This paper takes a direct approach to text sanitization.
Our insight is to consider both sensitivity and similarity via our new local DP notion. The sanitized texts also contribute to our sanitization-aware pretraining and fine-tuning, enabling privacy-preserving natural language processing
over the BERT language model with promising utility.
Surprisingly,  the high utility does not boost up the success rate of inference attacks.
%\footnote{Our code is at \href{https://github.com/xiangyue9607/SanText}{github.com/xiangyue9607/SanText}}
\end{abstract}

\section{Introduction}
Natural language processing (NLP) requires a lot of training data, which can be sensitive.
Na\"ive redaction approaches (\eg, removing common personally identifiable information)
is known to fail~\cite{jots/Sweeney15}: 
innocuous-looking fields can be linked to other information sources for reidentification.
The recent success of many language models (LMs) has motivated security researchers to devise advanced privacy attacks.
\citet{corr/abs-2012-07805} recover texts from (a single document of) the \emph{training data} via querying to an LM pretrained from it. 
\citet{sp/PanZJY20} and \citet{ccs/SongR20} target the text embedding, \eg, revealing from an encoded \emph{query} to an NLP service.

Emerging NLP works focus on only specific document-level (statistical) features~\cite{sigir/WeggenmannK18}
or producing private text representations~\cite{nips/XieDDHN17,emnlp/CoavouxNC18,emnlp/ElazarG18,acl/LiBC18} as initial solutions to the first issue above on training-data privacy. 
However, the learned representations are \emph{not} human-readable, which makes \emph{transparency} (\eg, required by GDPR)
questionable: an average user may not have the technical know-how to verify whether sensitive attributes have been removed or not.
Moreover, consider the whole NLP pipeline, the learned representations often entail extra modeling or non-trivial changes to existing NLP models, which take dedicated engineering efforts. 
%GDPR\href{https://gdpr-info.eu}{-info.eu}
%emnlp/LyuHL20,sigir/LyuLHX20 //dp-based methods.

\begin{figure}[t]
    \centering
    \includegraphics[width=\linewidth]{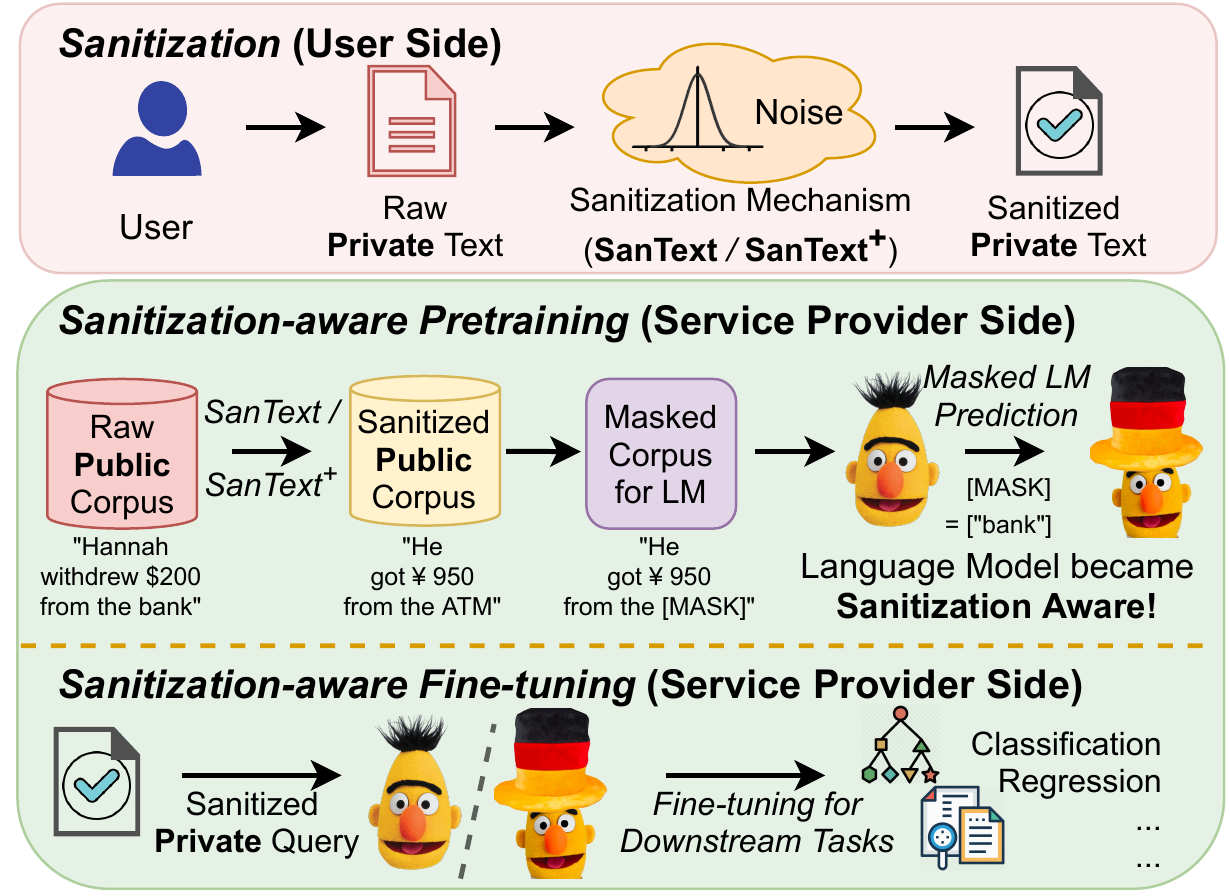}
    \caption{Workflow of our PPNLP pipeline, including the user-side sanitization and the service provider-side NLP modeling with pretraining/fine-tuning
    }
    \label{fig:model}
    \vspace{-15pt}
\end{figure}

\subsection{Sanitizing Sensitive Texts, Naturally}
With this state-of-affairs of the security and the NLP research,
we deem it better to address privacy from the root, \ie, directly producing sanitized text documents. 
Being the most native format, they incur minimal changes to existing NLP pipelines.
Being human-readable, they provide transparency (to privacy-concerning training-data contributors) and explainability 
(\eg, to linguists who might find the need for investigating how the training data contribute to a certain result). 
Moreover, it naturally extends the privacy protection to the inference phase. Users can apply our sanitization mechanism before sending queries (\eg, medical history) to the NLP service provider (\eg, diagnosis services). 

%\cite{ccs/AbadiCGMMT016}
Conceptually, we take a natural approach 
-- we sanitize text documents into also (sanitized) text documents.
This is in great contrast to the typical ``post-processing'' for injecting noises either to gradients in training (a deep neural network)~\cite{iclr/McMahanRT018}
or the ``cursed'' \emph{high-dimensional} text representations~\cite{emnlp/LyuHL20,sigir/LyuLHX20,wsdm/FeyisetanBDD20}.
It also leads to our $O(1)$ efficiency, freeing us from re-synthesizing the document word-by-word
via nearest neighbor searches over
the entire vocabulary space~$\mathcal{V}$~\cite{wsdm/FeyisetanBDD20}.

Technically, we aim for the \emph{de facto} standard of local differential privacy (LDP)~\citep{focs/DuchiJW13} to sanitize the user data \emph{locally}, based on which the service provider can build NLP models without touching any raw data. DP has been successful in many contexts, \eg, location privacy and survey statistics~\citep{ccs/AndresBCP13,uss/Murakami019}. However, DP text analytics appears to be a difficult pursuit (as discussed, also see Section~\ref{sect:related}), which probably explains why there are only a few works in DP-based text sanitization.
In high-level terms, text is rich in semantics, differentiating it from other more structured data.

Our challenge here is to develop \emph{efficient} and effective mechanisms that \emph{preserve the utility} of the text data with \emph{provable and quantifiable} privacy guarantees. 
Our insight is the formulation of a new LDP notion named \emph{Utility-optimized Metric LDP (UMLDP)}.
We attribute our success to the focus of UMLDP on protecting what matters (sensitive words) via ``sacrificing'' the privacy of non-sensitive (common) words. To achieve UMLDP, our mechanism directly samples noises on tokens.

Our result in this regard is already better than the state-of-the-art LDP solution producing sanitized documents~\cite{wsdm/FeyisetanBDD20} -- we got
$28$\% gain in accuracy
on the SST-2 dataset~\cite{iclr/WangSMHLB19} on average at the same privacy level 
(\ie, the same LDP parameter)
while being much more efficient (${\sim}60\times$ faster, precomputation included).

\subsection{Privacy-Preserving NLP, Holistically}
Text sanitization is essential but just one piece of the whole privacy-preserving NLP (PPNLP) pipeline. 
While most prior works in text privacy are motivated by producing useful data for some downstream tasks, the actual text analytics are hardly explored, not to say in the context of many recent general-purpose language models.
As simple as it might seem, we start to see design choices that can be influential. 
Specifically, our challenge here is to adapt the currently dominating pretraining-fine-tuning paradigm (\eg, BERT~\cite{devlin2019bert}) over sanitized texts for building the model.

Our design is to build in privacy at the root again, in contrast to the afterthought approach.
We found it beneficial to sanitize even the public data before feeding them to training. 
It is not for protecting the public data per se.
The intuition here is that it ``prepares'' the model to work with sanitized queries, 
which explains our eventual (slight) increase in accuracy while \emph{additionally ensuring privacy}.

Specifically, we propose a sanitization-aware pretraining procedure (Figure~\ref{fig:model}). We first use our mechanisms to sanitize the public texts, mask the sanitized texts (as in BERT), and train the LM by predicting a \texttt{MASK} position as its \emph{original unsanitized token}. LMs preptrained with our sanitization-aware procedure are expected to be more robust to noises in the sanitized texts and achieve better utility when fine-tuning on downstream tasks.

We conduct experiments on three representative NLP tasks to empirically confirm that our proposed PPNLP pipeline preserves both utility and privacy. It turns out that our sanitization-based pretraining (using only $1/6$ of data used in the original BERT pretraining) can even improve the utility of NLP tasks while maintaining privacy comparable to the original BERT. 
Note that there is an inherent tension between utility and privacy, and privacy attack is also inference in nature. 
To empirically demonstrate the privacy aspect of our pipeline, \ie, it does not make our model a more powerful tool helping the attacker, 
we also conduct the ``mask token inference" attack on private texts, 
which infers the masked token given its context based on BERT.
As a highlight, our base solution
\baseDP improves the defense rate by $20\%$ with only a $4\%$ utility loss on the SST-2 dataset. 
We attribute our surprising
result of mostly helping only good guys to our natural approach: to avoid the model memorizing sensitive texts ``too well,'' we fed it with sanitized text.

\section{Related Work}\label{sect:related}
\noindent\textbf{Privacy risks in NLP.}
A taxonomy of attacks that recover sensitive attributes or partial raw text from text embeddings output by popular LMs has been proposed~\cite{ccs/SongR20}, without any assumptions on the structures or patterns in input text.
%Very recent works~\cite{ccs/SongR20,sp/PanZJY20} show that sensitive attributes or partial raw text can be recovered from text embeddings output by popular LMs.
%Even worse, 
\citet{corr/abs-2012-07805} also show a powerful black-box attack on 
\mbox{GPT-2}~\citep{radford2019language} that extracts verbatim texts of training data.
Defense with rigorous guarantees (DP) is thus vital.

\smallskip
\noindent\textbf{Differential privacy and its application in NLP.}
DP~\citep{icalp/Dwork06} has emerged as the \emph{de facto} standard for statistical analytics~\citep{uss/WangBLJ17,sp/WangLJ18,sigmod/CormodeKS18}.
\iffalse
It has also been applied in deep learning by adding noise to gradients during training~\cite{iclr/McMahanRT018}.
%ccs/ErlingssonPK14,ccs/AbadiCGMMT016
%, thus providing privacy about the training data

\fi
A few efforts inject high-dimensional DP noise into text representations~\citep{icdm/FeyisetanDD19,wsdm/FeyisetanBDD20,emnlp/LyuHL20,sigir/LyuLHX20}.
The noisy representations are not human-readable
and not directly usable by existing NLP pipelines, \ie, they consider a different problem not directly comparable to ours.
% (\eg, it takes dedicated engineering efforts to customize downstream modules/layers to process such private representation).
%Thus,
%\citet{icdm/FeyisetanDD19,wsdm/FeyisetanBDD20} further project a noisy token embedding back to a token via a nearest neighbor search on the entire vocabulary space.
More importantly, they fail to strike a nice privacy-utility balance due to ``the curse of dimensionality,'' \ie, the magnitude of the noise is too large for high-dimensional token embedding, and thus it becomes exponentially less likely to find a noisy embedding close to a real one on every dimension.
This may also explain why an earlier work focuses on document-level statistics only, \eg, term-frequency vectors~\cite{sigir/WeggenmannK18}.

Our approaches produce natively usable sanitized texts via directly sampling a substitution for each token from a precomputed distribution (to be detailed in Section~\ref{sect:ppnlp}), circumventing the dimension curse and striking a privacy-utility tradeoff while being much more efficient.
A concurrent work~\cite{corr/abs-2104-07504} also considers the whole NLP pipeline, but it still builds on the token-projection approach~\cite{wsdm/FeyisetanBDD20}.
%privacy-preserving NLP pipeline as a whole, 

\smallskip
\noindent\textbf{Privacy-preserving text representations.} 
Learning private text representations via adversarial training is also an active area~\citep{nips/XieDDHN17,emnlp/CoavouxNC18,emnlp/ElazarG18,acl/LiBC18}.
An adversary is trained to infer sensitive information jointly with the main model, while the main model is trained to maximize the adversary’s loss and minimize the primary learning objective. While we share the same general goal, 
our aim is not such representations (similar to those with DP) but to release sanitized text for general purposes.

\section{Defining (Local) Differential Privacy}
\label{sect:pre}
Suppose each user holds a document $D = \langle x_i \rangle_{i = 1}^L$ of $L$ tokens (which can be a character, a subword, a~word, or an n-gram), where $x_i$ is from a vocabulary $\V$ of size $|\V|$.
For privacy, each user derives a sanitized version $\hat{D}$ by running a common text sanitization mechanism $\M$ over $D$ on local devices.
Specifically, $\M$ works by replacing every token $x_i$ in $D$ with a substitution $y_i \in \V$, assuming that $x_i$ itself is unnecessary for NLP tasks while its semantics should be preserved for high utility.
The output $\hat{D}$ is then shared with an NLP service provider. 

We consider a typical threat model in which each user does not trust any other party and views them as an attacker with access to $\hat{D}$ in conjunction with any auxiliary information (including $\M$).

\subsection{(Variants of) Local Differential Privacy}
Let $\X$ and $\Y$ be the input and output spaces.
A randomized mechanism $\M: \X \rightarrow \Y$ is a probabilistic function that assigns a random output $y \in \mathcal{Y}$ to an input $x \in \mathcal{X}$.
Every $y$ induces a probability distribution on the underlying space.
For sanitizing text, we set both $\X$ and $\Y$ as the vocabulary $\V$.

\begin{definition}[$\epsilon$-LDP~\citep{focs/DuchiJW13}]
\label{def:ldp}
Given a privacy parameter $\epsilon \geq 0$, $\mathcal{M}$ satisfies $\epsilon$-local differential privacy ($\epsilon$-LDP) if, for any $x, x', y \in \V$,
\begin{align*}
  \Pr[\mathcal{M}(x) = y] \leq e^\epsilon \cdot \Pr[\mathcal{M}(x') = y].
\end{align*}
\end{definition}

Given an observed output $y$, from the attacker's view, the likelihoods $y$ is derived from $x$ and $x'$ are similar.
A smaller $\epsilon$ means better privacy due to a higher indistinguishability level of output distributions, yet the outputs retain less utility.

$\epsilon$-LDP is a very strong privacy notion for its homogeneous protection over all input pairs.
However, this is also detrimental to the utility: no matter how unrelated $x$ and $x'$ are, their output distributions must be similar.
As a result, a sanitized token $y$ may not (approximately) capture the semantics of its input $x$, degrading the downstream tasks.

\smallskip
\noindent\textbf{LDP over metric spaces.} 
To capture semantics, we borrow the relaxed notion of Metric-LDP (MLDP)~\citep{csfw/Alvim0PP18} originally proposed for location privacy~\citep{ccs/AndresBCP13} with the distance metric $d(\cdot, \cdot)$ between two locations (\eg, Manhattan distance~\cite{pet/ChatzikokolakisABP13}).

\begin{definition}[MLDP]
\label{def:mldp}
Given $\epsilon \geq 0$ and a distance metric $d: \V \times \V \rightarrow \mathbb{R}_{\geq 0}$ over $\V$, $\mathcal{M}$ satisfies MLDP or $\epsilon \cdot d(x,x')$-LDP if, for any $x,x',y \in \V$, 
\begin{align*}
  \Pr[\mathcal{M}(x) = y] \leq e^{\epsilon \cdot d(x,x')} \cdot \Pr[\mathcal{M}(x') = y].
\end{align*}
\end{definition}

When $d(x,x') = 1~\forall x \neq x'$, MLDP becomes LDP.
For MLDP, the indistinguishability of output distributions is further scaled by the distance between the respective inputs.
Roughly, the effect of $\epsilon$ becomes ``adaptive.''
To apply MLDP, one needs to carefully define the metric $d$
(see Section~\ref{sect:base}).
%(see Section~\ref{sect:ppnlp}).

\smallskip
\noindent\textbf{Incorporating ULDP to further improve utility.}
Utility-optimized LDP~\citep{uss/Murakami019} (ULDP) also relaxes LDP, which was originally proposed for aggregating ordinal responses.
It exploits the fact that different inputs have different sensitivity levels to achieve higher utility.
By assuming that the input space is split into \emph{sensitive} and \emph{non-sensitive} parts, ULDP achieves a privacy guarantee equivalent to LDP for \emph{sensitive} inputs.

In our context, more formally speaking, let $\V_S \subseteq \V$ be the set of sensitive tokens common to all users, and $\V_N = \V \setminus \V_S$ be the set of remaining tokens.
The output space $\V$ is split into the \emph{protected} part $\V_P \subseteq \V$ and the \emph{unprotected} part $\V_U = \V \setminus \V_P$.

The image of $\V_S$ is restricted to $\V_P$, \ie,
a sensitive $x \in \V_S$ can only be mapped to a protected $y \in \V_P$.
For text, we can set $\V_S = \V_P$ for simplicity.
%as every sensitive token needs to be protected.
While a non-sensitive $x \in V_N$ can be mapped to $\V_P$, every $y \in \V_U$ must be mapped from 
%a non-sensitive $x \in$
$\V_N$, 
which helps to improve the utility.

\subsection{Our New Utility-optimized MLDP Notion}
Among many variants of (L)DP notions, we found the above two %recently formalized LDP notions 
variants 
(\ie, ULDP and MLDP) provide useful insight in quantifying semantics and privacy of text data. 
We thus formulate the new privacy notion of utility-optimized MLDP (UMLDP).
%defined below. 

\iffalse
Given $\V_S = \V_P \subseteq \V$, $\V_N = \V_U = \V \setminus \V_S$, two privacy parameters $\epsilon, \epsilon_0 \geq 0$, and a distance metric $d: \V \times \V \rightarrow \mathbb{R}_{\geq 0}$, $\M$ satisfies $(\V_S, \epsilon, \epsilon_0)$-UMLDP, if 
\begin{align*}
\forall x \in \V_S
& : 
\Pr[\mathcal{M}(x)=y]
\begin{cases}
>0 & \forall y \in \V_P, \\
=0 & \forall y \in \V_U;
\end{cases}
\\
\forall x \in \V_N
& : 
\Pr[\mathcal{M}(x)=y]
\begin{cases}
>0 & \forall y \in \V_P \cup \{x\}, \\
=0 & \forall y \in \V_U \setminus \{x\};
\end{cases}
\end{align*}
$\forall x,x' \in \V$ and $\forall y \in \V_P$, 
$$\Pr[\mathcal{M}(x)=y] \leq e^{\epsilon d(x,x') + \epsilon_0} \Pr[\mathcal{M}(x')=y].$$
\fi

\begin{definition}[UMLDP]
\label{def:umldp}
Given $\V_S \cup \V_N = \V$, two privacy parameters $\epsilon, \epsilon_0 \geq 0$, and a distance metric $d: \V \times \V \rightarrow \mathbb{R}_{\geq 0}$, $\M$ satisfies $(\V_S, \V_P, \epsilon, \epsilon_0)$-UMLDP, if
\\\smallskip\noindent
i) for any $x, x' \in \V$ and any $y \in \V_P$, we have 
$$\Pr[\mathcal{M}(x)=y] \leq e^{\epsilon d(x,x') + \epsilon_0} \Pr[\mathcal{M}(x')=y];$$
ii) for any $y \in \V_U$, \ie,
from an unprotected 
set $\V_U$ where $\V_U \cap \V_P = \emptyset$, there is an $x \in \V_N$ such that 
\begin{align*}
\Pr[\mathcal{M}(x)=y] & > 0, 
%\ \V_P \cap \V_U = \emptyset,
\\ \Pr[\mathcal{M}(x')=y] & = 0 
\ \forall x' \in \V \setminus \{x\}.
\end{align*}
\end{definition}

\begin{figure}[!t]
\centering
\includegraphics[width=0.45\linewidth]{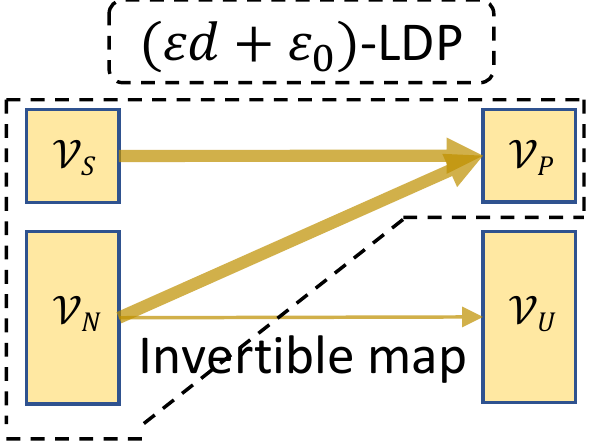}
\vspace{-2mm}
\caption{Overview of our new UMLDP notion}
\vspace{-5mm}
\label{fig:metric_uldp}
\end{figure}

Figure~\ref{fig:metric_uldp} summarizes the treatment of UMLDP.
It exhibits ``invertibility,'' \ie, $y \in \V_U$ must be ``noise-free'' and mapped deterministically.
Apart from generalizing $\epsilon$ in the ULDP definition 
(recalled in Appendix~\ref{sect:uldp})
into $\epsilon d(x,x')$, 
we incorporate an additive bound~$\epsilon_0$ due to the invertibility, 
which makes the derivation of $\epsilon$ easier. 
%without making it ``too loose.''
Looking ahead, $\epsilon_0$ would appear naturally in the analysis of our UMLDP mechanism for the invertible case.
%assigns probability $(1-p)$ to the invertible mapping.
\iffalse
Our formulation is slightly different but still captures the original ULDP definition.
%(recalled in Appendix~\ref{sect:uldp}), 
Particularly, we incorporate an additive bound~$\epsilon_0$ 
%(see Appendix~\ref{sect:uldp}'s discussion).
(discussed in Appendix~\ref{sect:uldp}).
\fi

UMLDP (and MLDP), as an LDP notion, satisfies the \emph{composability} and \emph{free post-processing}.
The former means that the sequential execution of $\epsilon_1$-LDP and $\epsilon_2$-LDP mechanisms satisfies $(\epsilon_1+\epsilon_2)$-LDP, \ie, $\epsilon$ can be viewed as the privacy ``budget” of a sophisticated task comprising multiple subroutines, each consumes a part of $\epsilon$ such that their sum equals $\epsilon$.
The latter means further processing the mechanism outputs incurs no extra privacy loss.

\section{Our Privacy-Preserving NLP Pipeline}
\label{sect:ppnlp}
\subsection{Overview}
We propose two token-wise sanitization methods with (U)MLDP: \baseDP and \baseDPp, 
which build atop a variant of the exponential mechanism (EM)~\citep{focs/McSherryT07} 
over the ``native'' text tokens as both input and output spaces to avoid going to the ``cursed dimensions'' of token embeddings.
EM samples a replacement~$y$ for an input $x$ based on an exponential distribution, with more ``suitable'' $y$'s sampled with higher probability (detailed below).
It is well-suited for (U)MLDP by considering the ``suitability'' as how well the semantics of $x$ is preserved for the downstream tasks (run over the sanitized text $y$) to remain accurate.

%It is reasonable as many common tokens (such as ``the'' and ``it'') are non-sensitive for all users.
To quantify this, we utilize an embedding model mapping tokens into a real-valued vector space.
The semantic similarity among tokens can then be measured via the Euclidean distance between their corresponding vectors.
Our base design \baseDP outputs $y$ with probability inverse proportional to the distance between $x$ and $y$: the shorter the distance, the more semantically similar they are.
\baseDPp considers some tokens $\V_N$ in $\V$ are non-sensitive, and runs \baseDP over 
the sensitive part $\V_S = V \setminus \V_N$
(\ie, it degenerates to \baseDP if $\V_S = \V$).
For $\V_N$, we tailor a probability distribution to provide UMLDP as a whole.

With \baseDP or \baseDPp, each user sanitizes $D$ into $\hat{D}$ and uploads it to the service provider for performing any NLP task built atop a pretrained LM, \eg, BERT.
Typically, the task pipeline consists of an embedding layer, an encoder module, and task-specific layers, \eg, for classification.

Without the raw text, 
the utility can degrade;
we thus propose two approaches for improving it.
The first one is to pretrain only the encoder on the sanitized public corpus to adapt to the noise.
It is optional if pretraining is deemed costly.
The second is to fine-tune the full pipeline on $\hat{D}$'s, which updates both the encoder and task layers.

\subsection{Base Sanitization Mechanism: \baseDP}
\label{sect:base}

\setlength{\textfloatsep}{0pt}
\begin{algorithm}[t!]
\caption{Base Mechanism \baseDP}
\label{alg:baseline}
\LinesNumbered
\KwIn{A private document $D = \langle x_i \rangle_{i = 1}^L$, and a privacy parameter $\epsilon \geq 0$}
\KwOut{Sanitized document $\hat{D}$}
Derive token vectors $\phi(x_i)$ for $i \in [1,L]$\;
\For{$i = 1, \ldots, L$}{
Run $\M(x_i)$ to sample a sanitized token $y_i$ with probability defined in Eq.~(\ref{eq:baseline})\; 
}
Output sanitized $\hat{D}$ as $\langle y_i \rangle_{i = 1}^L$\;
\end{algorithm}

%The original text space is mapped into a Euclidean space, and t
In NLP, a common step is to employ an embedding model\footnote{We assume that it has been trained on a large public corpus and shared by all users.}
%and its parameters are fixed.
mapping semantically similar tokens to close vectors in a Euclidean space. Concretely, an embedding model is an injective mapping $\phi: \V \rightarrow \mathbb{R}^m$, for dimensionality $m$.
The distance between any two tokens $x$ and $x'$ can be measured by the Euclidean distance of their embeddings: $d(x,x') = d_{\euc}(\phi(x),\phi(x'))$.
As $\phi$ is injective, $d$ satisfies the axioms of a distance metric.

Algorithm~\ref{alg:baseline} lists the pseudo-code of \baseDP for sanitizing a private document $D$ at the user side.
The first step is to use $\phi$ to derive token embeddings of each token\footnote{For easy presentation, we omit the subscript $i$ later.} 
$x$ in $D$.
Then, for each $x$, we run $\M(x)$ to sample a sanitized $y$
with probability
\begin{align}
\label{eq:baseline}
 \Pr[\M(x)=y] = C_{x} \cdot e^{-\frac{1}{2} \epsilon \cdot d_{\euc}(\phi(x),\phi(y))}
\end{align}
where 
$C_x = (\sum_{y' \in \V}e^{-\frac{1}{2} \epsilon \cdot d_{\euc}(\phi(x),\phi(y'))})^{-1}$.

The smaller $d_{\euc}(\phi(x),\phi(y))$, the more likely $y$ is to replace $x$. 
To boost the sanitizing efficiency, we can precompute a $|\V| \times |\V|$ probability matrix, where each entry $(i,j)$ denotes the probability of outputting $y_j$ on input $x_i$, upon obtaining $\phi(x)$ for $\forall x \in \V$.
Lastly, the sanitized $\hat{D} = \langle y_i \rangle_{i = 1}^L$ can be released to the service provider for NLP tasks.

\subsection{Enhanced Mechanism: \baseDPp}
\label{sect:enhanced}

\setlength{\textfloatsep}{0pt}
\begin{algorithm}[t!]
\caption{Enhanced \baseDPp}
\label{alg:enhanced}
\LinesNumbered
\KwIn{A private document $D = \langle x_i \rangle_{i = 1}^L$, a privacy parameter $\epsilon \geq 0$, probability $p$ for a biased coin, and sensitive $\V_S$}
\KwOut{Sanitized document $\hat{D}$}
Derive token vectors $\phi(x_i)$ for $i \in [1,L]$\;
\For{$i = 1, \ldots, L$}{
\eIf{$x_i \in \V_S$}{
Sample a substitution $y_i \in\V_P = \V_S$ with probability given in Eq.~(\ref{eq:baseline})\algorithmiccomment{Run \baseDP over $\V_S$ and $\V_P$}\;}
{Output $y_i = x_i$ with prob. $(1-p)$; or $y_i \in\V_P$ with prob. in Eq.~(\ref{eq:enhanced});}
}
Output sanitized $\hat{D}$ as $\langle y_i \rangle_{i = 1}^L$\;
\end{algorithm}

In \baseDP, all tokens in $\V$ are treated as sensitive, which leads to excessive protection and utility loss. 
Following the less-is-more principle,
we divide $\V$ into $\V_S$ and $\V_N$, and focus on protecting~$\V_S$.

Observing that most frequently used tokens (\eg, a/an/the) are 
non-sensitive to virtually all users, we use token frequencies for division. 
A simple strategy, which is also used in our experiments, is to mark the top $w$ of low-frequency tokens (according to a certain corpus) as $\V_S$, where $w$ is a tunable parameter.
Looking ahead, this ``basic'' method already showed promising results. 
(Further discussion can be found in Section~\ref{sect:sensitivity}).

Algorithm~\ref{alg:enhanced} lists the pseudo-code of \baseDPp with $\V_S=\V_P$ and $\V_N=\V_U$ shared by all users.
The first step,
as in \baseDP, 
is to derive the token embeddings in $D$.
Then, for each token $x$, if it is in $\V_S$, we sample its substitution $y$ from $\V_P$ with probability given in Eq.~(\ref{eq:baseline}).
(This is equivalent to running \baseDP over $\V_S$ and $\V_P$.)
For $x \in \V_N$, we toss a biased coin. With probability $(1 - p)$, we output $y$ as $x$ (\ie, the ``invertibility''). 
Otherwise, 
we sample $y \in \V_P$ with probability
\begin{align}
\label{eq:enhanced}
 \Pr[\M(x)=y] = p \cdot C_{x} \cdot e^{-\frac{1}{2} \epsilon \cdot d_{\euc}(\phi(x),\phi(y))}
 %.\hspace{-4pt}
\end{align}
where $C_x = (\sum_{y' \in \V_P}e^{-\frac{1}{2} \epsilon \cdot d_{\euc}(\phi(x),\phi(y'))})^{-1}$.
%is for normalization.

As in \baseDP, we can also precompute two $|\V_S| \times |\V_P|$ and $|\V_N| \times |\V_P|$ probability matrices, which correspond to Eq.~(\ref{eq:baseline}) and (\ref{eq:enhanced}), for optimizing the sanitizing efficiency.
Lastly, the sanitized $\hat{D}$ of $\langle y \rangle_{i = 1}^L$ can be released to the service provider.

\begin{theorem}
\label{thm:base}
Given $\epsilon \geq 0$ and $d_{\euc}$ over the embedding space $\phi$ of~$\V$, \baseDP satisfies MLDP.
\end{theorem}

\begin{theorem}
\label{thm:enhanced}
Given $(\V_S = \V_P) \subseteq \V$, $\epsilon \geq 0$, $\epsilon_0 = \ln{\frac{1}{p}} \geq 0$, and $d_{\euc}$ over the embedding space $\phi$ of~$\V$, \baseDPp satisfies $(\V_S, \V_P, \epsilon, \epsilon_0)$-UMLDP.
\end{theorem}

Their proofs are in Appendix~\ref{sect:proof}.

\subsection{NLP over Sanitized Text}
\label{sect:nlp}
With $\hat{D}$'s (shared by the users), the service provider can perform any NLP task.
In this work, we focus on those built on a pretrained LM, and in particular, we study BERT as an example due to its wide adoption and superior performance.
The full NLP pipeline is deployed at the service provider.

Given a piece of (sanitized) text, the embedding layer maps it to a sequence of token embeddings.
The encoder computes a sequence representation from the token embeddings, allowing task-specific layers to make predictions.
For example, the task layer could be a feed-forward neural network for multi-label classification of a diagnosis system.

The injected noise deteriorates the performance of downstream tasks as the service provider cannot access the raw texts $\{D\}$.
To mitigate this, we propose two approaches -- pretraining the encoder and fine-tuning the full pipeline, which allow the tasks to be ``adaptive'' to the noise to some extent.

\smallskip
\noindent\textbf{Pretraining BERT over sanitized public corpus.} Besides $\hat{D}$'s, the service provider can also obtain a massive amount of text that is publicly available (say, the English Wikipedia).
It also has access to the sanitization mechanisms, and it can produce the sanitized public text (as how users produce $\hat{D}$'s).

Our key idea is to let the service provider pretrain the encoder (\ie, BERT) over the sanitized public text, making it more ``robust'' in handling $\hat{D}$'s.
We thus initialize the encoder with the original BERT checkpoint and conduct further pretraining with an adapted masked language model (MLM) loss. 
In more detail, the adapted MLM objective is to predict the \emph{original masked tokens} given the sanitized context instead of the one from the raw public text.
We note that this is beneficial for improving the task utility, yet may breach the user privacy as the objective learns to ``recover'' the original tokens or semantics.
In Section~\ref{sect:exp_pretrain}, our results will show that such pretrained BERT indeed improves accuracy, with comparable privacy as in original BERT.

\smallskip
\noindent\textbf{Fine-tuning the full NLP pipeline.}
After pretraining BERT using sanitized public text, the service provider can further improve the efficacy of downstream tasks by fine-tuning the full pipeline.
We assume that the ground-truth labels are available to the service provider, say, 
inferring from $\hat{D}$'s when they can preserve similar semantics to the raw text.
Then, the sanitized text-label pairs are used for training/fine-tuning downstream task models, with gradients back-propagated to update the parameters of both the encoder and task layer. 
We leave more realistic/complex labeling processes based on sanitized texts as future work.

\subsection{Definition of ``Sensitivity''}\label{sect:sensitivity}
Simply treating the top $w$ of least frequent tokens (\eg, according to a public reference corpus) as the sensitive token set already led to promising results (see Section~\ref{sect:exp_comparison}).
By this definition, stop words are mostly non-sensitive (\eg, for $w=0.9$ over the sentiment classification dataset we used, ${\sim}98\%$ of the stop words are deemed non-sensitive).
For context-specific corpus, this strategy is better than merely using stop words, 
\eg, breast cancer becomes non-sensitive among breast-cancer patients.
%(SST-2) 

Sophisticated machine-learning approaches or other heuristics could also be considered, 
\eg, training over context-specific reference corpus
or identifying tokens with personal (and hence sensitive) information
(\eg, names).
We leave as future work.

Moreover, the definition of sensitivity may vary across users. Some may consider a token deemed non-sensitive by most other users sensitive.
The original ULDP work~\citep{uss/Murakami019} has discussed a personalized mechanism that preprocesses such tokens by mapping them to a set of semantic tags, which are the same for all users. These tags will be treated as sensitive tokens for the ULDP mechanism. Apparently, this approach is application-specific and may not be needed in some applications; hence we omit it in this work.

%!TEX root = main.tex

\begin{table*}[t]
\resizebox{\linewidth}{!}{%
\begin{tabular}{l|c|c|c|c|c|c|c|c|c}
\hline
\multirow{2}{*}{\textbf{Mechanisms}} & \multicolumn{3}{c|}{\textbf{SST-2}} & \multicolumn{3}{c|}{\textbf{MedSTS}} & \multicolumn{3}{c}{\textbf{QNLI}} \\ \cline{2-10} 
 &   $\epsilon=1$ & $\epsilon=2$ & $\epsilon=3$ & $\epsilon=1$ & $\epsilon=2$ & $\epsilon=3$ & $\epsilon=1$ & $\epsilon=2$ & $\epsilon=3$ \\ \hline\hline
 Random &
$0.4986$ &	$0.4986$	&$0.4986$	&	$0.0196$	&	$0.0196$	&	$0.0196$ & $0.5152$  & $0.5152$ &$0.5152$
\\ \cline{1-10} 
\citet{wsdm/FeyisetanBDD20} &$0.5099$  &$0.5143$  &$0.5345$  &$0.0201$  &$0.0361$  &$0.0452$ &$0.5162$  & $0.5256$ &$0.5333$ \\ \cline{1-10} 
  \baseDP &
$0.5101$ &	$0.5838$	&	$0.8374$	&	$0.0351$	&	$0.5392$	&	$0.8159$ & $0.5372$  & $0.5598$  &$0.8116$
\\ \cline{1-10} 
 \baseDPp & $\mathbf{0.7796}$ & $\mathbf{0.7943}$ & $\mathbf{0.8516}$ & $\mathbf{0.4965}$ & $\mathbf{0.7082}$ & $\mathbf{0.8162}$ & $\mathbf{0.7699}$ & $\mathbf{0.7760}$ & $\mathbf{0.8131}$\\ \hline \hline
Unsanitized  & \multicolumn{3}{c|}{$0.9251$} & \multicolumn{3}{c|}{$0.8527$} & \multicolumn{3}{c}{$0.9090$} \\ \hline
\end{tabular}
}
\vspace{-5pt}
\caption{Utilities comparison of sanitization mechanisms under similar privacy levels using the GloVe embedding
%For a fair comparison,  we use the GloVe \cite{emnlp/PenningtonSM14} embedding for our methods following \cite{wsdm/FeyisetanBDD20}
%: Our context-aware sanitized mechanism \baseDP achieve the best performance in all the evaluation settings of the three tasks.
}
\label{tbl:comparisonTask}
\vspace{-15pt}
\end{table*}

\section{Experiments}
\subsection{Experimental Setup}
We consider three representative downstream NLP tasks (datasets) with privacy implications.

\smallskip
\noindent
\textbf{Sentiment Classification (SST-2)}. When people write online reviews, especially the negative ones, they may worry about having their identity traced via writing too much that may provide hints of authorship or linkage to other online writings.
For this task, we use the preprocessed version in GLUE benchmark~\cite{iclr/WangSMHLB19} of  (binary) Stanford Sentiment Treebank (SST-2) dataset~\cite{emnlp/SocherPWCMNP13}. \emph{Accuracy} (w.r.t. the ground truth included in the dataset) is used as the evaluation metric.

\smallskip
\noindent
\textbf{Medical Semantic Textual Similarity (MedSTS).} Automated processing of patient records is a significant research direction, and one such task is computing the semantic similarity between clinical text snippets for the benefit of reducing the cognitive burden. We choose a very recent MedSTS dataset~\cite{wang2018medsts} for this task, which assigns a numerical score to each pair of sentences, indicating the degree of similarity. We report the \emph{Pearson correlation coefficient} (between predicted similarities and human judgments) for this task.

\smallskip
\noindent
\textbf{Question Natural Language Inference (QNLI)}. Question-answering (QA) aims to automatically answer user questions based on documents. We consider a simplified setting of QA, namely QNLI, which predicts whether a given document contains the answer to the question. We use the QNLI dataset from GLUE benchmark~\cite{iclr/WangSMHLB19}.

We implement our sanitized mechanisms using Python and the sanitization-aware training using the Transformers library~\cite{Wolf2019HuggingFacesTS}. We use sanitized data to train and test prediction models for all three tasks.
We either build vocabularies for the tasks using GloVe embeddings~\cite{emnlp/PenningtonSM14} or adopt the same BERT vocabulary~\cite{devlin2019bert}. 
Table~\ref{tbl:vocabSize} 
%(in Appendix) 
shows their sizes.
Our sanitization-aware pretraining uses WikiCorpus (English version, a 2006 dump, $600$M words)~\cite{reese2010wikicorpus}. 
We start from the \texttt{bert-base-uncased} (instead of randomly initialized) model to accelerate the pretraining. 

We set the maximum sequence length to $512$, training epoch to $1$, batch size to $6$, 
learning rate to 5e-5, warmup steps to $2000$, and MLM probability to $0.15$. 
Our sanitization-aware fine-tuning uses the \texttt{bert-base-uncased} model for \mbox{SST-2}/QNLI, and \texttt{ClinicalBERT}~\cite{alsentzer2019publicly} for MedSTS. 
We set the maximum sequence length to $128$, training epochs to $3$, 
batch size to $64$ for SST-2/QNLI or $8$ for MedSTS, 
and learning rate to 2e-5 for SST-2/QNLI or 5e-5 for MedSTS.
Other hyperparameters are kept default.
Our hyperparameters followed the transformer library~\cite{Wolf2019HuggingFacesTS} and popular setups in the original dataset literature~\cite{iclr/WangSMHLB19,wang2018medsts}.

\subsection{Comparison of Sanitization Mechanisms}
\label{sect:exp_comparison}
We first compare our \baseDP and \baseDPp with random sanitization and the state-of-the-art of~\citeauthor{wsdm/FeyisetanBDD20} (FBDD). 
Here, we use the GloVe embedding as in FBDD
for a fair comparison.
%\footnote{Our other experiments will use the BERT embedding.}.
Random sanitization picks a token from the vocabulary uniformly.
We set the UMLDP parameters $p=0.3, w=0.9$ for \baseDPp 
(while Figure~\ref{fig:santext+w_p} plots the impacts of $p$ and $w$ when fixing $\epsilon=2$).

%While, FBDD adds multivariate Gaussian noise to each word vector and then replaces a word by one whose vector is closest to the noisy one. 

%We use the GloVe embeddings \cite{emnlp/PenningtonSM14} as word vectors for all the sanitization mechanism following the original setting \cite{wsdm/FeyisetanBDD20}. 

%We report the results of sanitized-aware fine-tuning on the sanitized data using \texttt{bert-base-uncased}. 

\begin{figure}[t]
    \centering
    \includegraphics[width=0.48\linewidth]{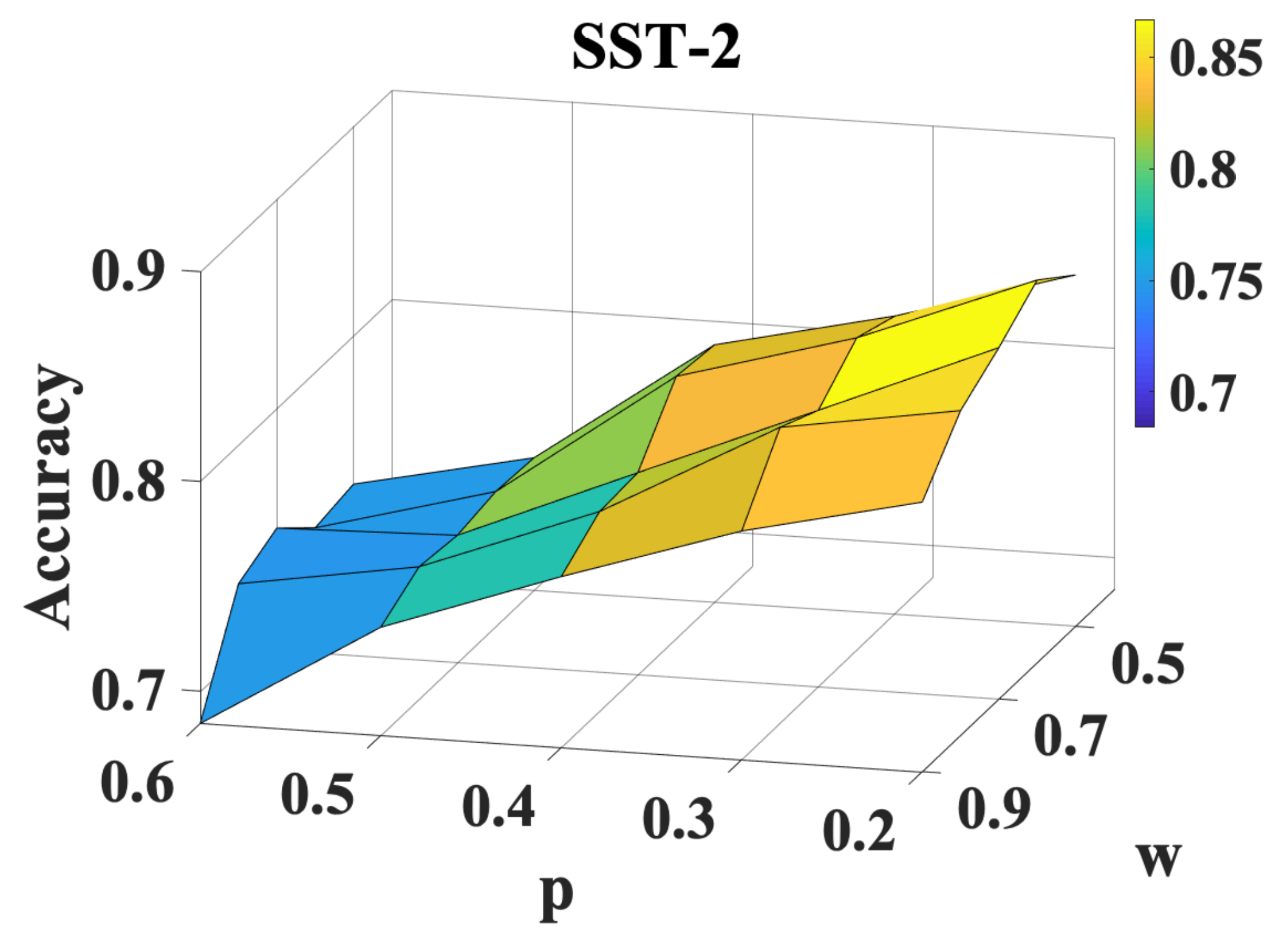}
    \includegraphics[width=0.48\linewidth]{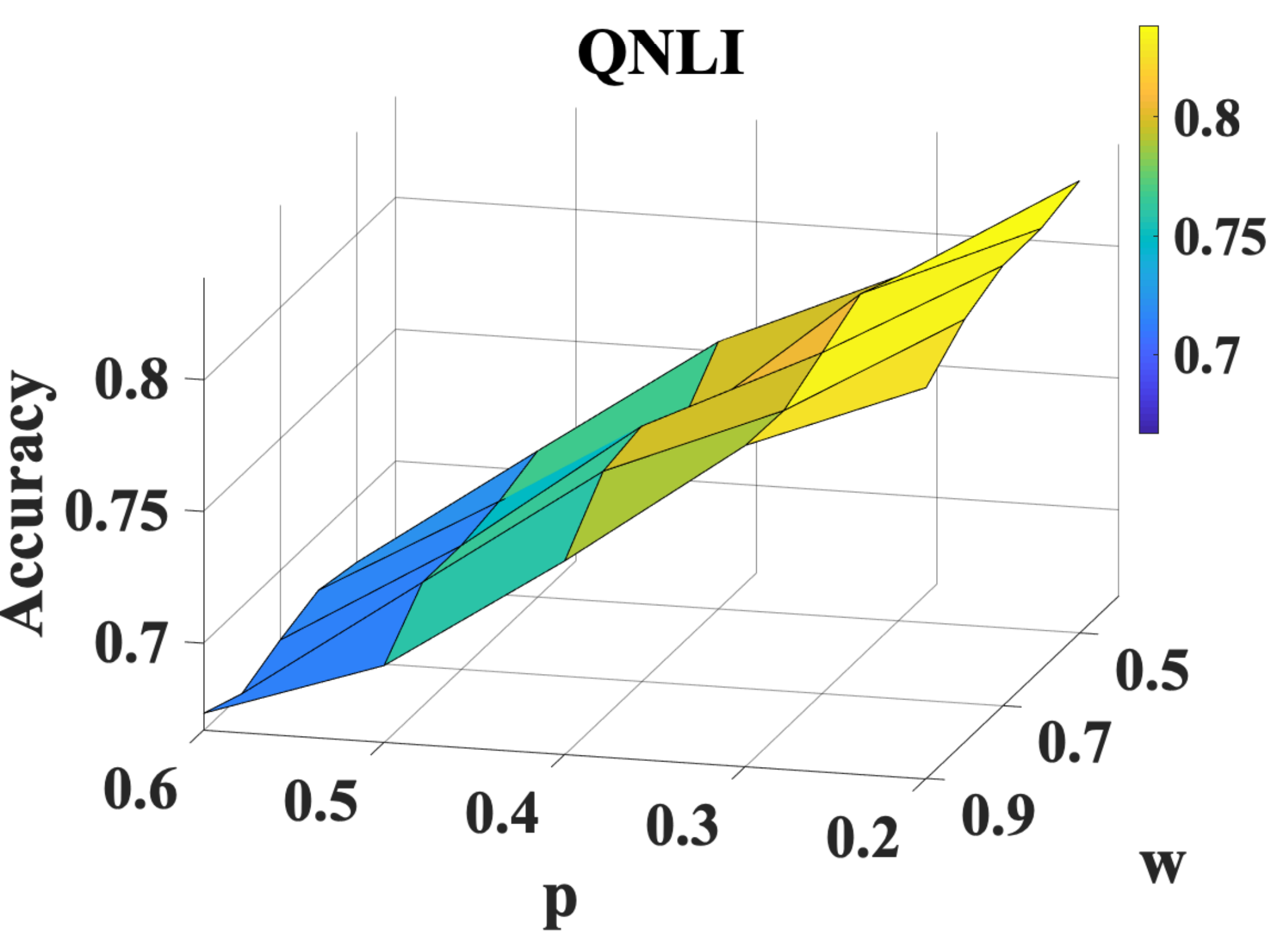}
    \vspace{-7pt}
    \caption{Performance of \baseDPp over $(w, p)$ when fixing $\epsilon=2$ based on the GloVe embedding
    }
    \label{fig:santext+w_p}
\end{figure}

Table~\ref{tbl:comparisonTask} shows the utility of the four mechanisms for the three selected tasks at different privacy levels. FBDD has a higher utility than random replacements. While both FBDD and \baseDP are based on word embeddings, \baseDP does not suffer from the ``curse-of-dimensionality" and achieves better utility at the same privacy level. \baseDPp achieves the best utilities in all cases since it allows the non-sensitive tokens to be noise-free, lowering the noise and improving the utility. 

In terms of efficiency, our \baseDP and \baseDPp are very efficient (\eg, ${\sim}2$min for the SST-2 dataset) compared with FBDD (${\sim}117$min) when they all run on a 24 core CPU machine. This is because our mechanisms only need to compute the sampling probability once and use the same probability matrix for sampling each time, while FBDD needs to recalculate the additive noise and re-search the nearest neighbor each time.

\begin{figure}[t]
    \centering
    \includegraphics[width=\linewidth]{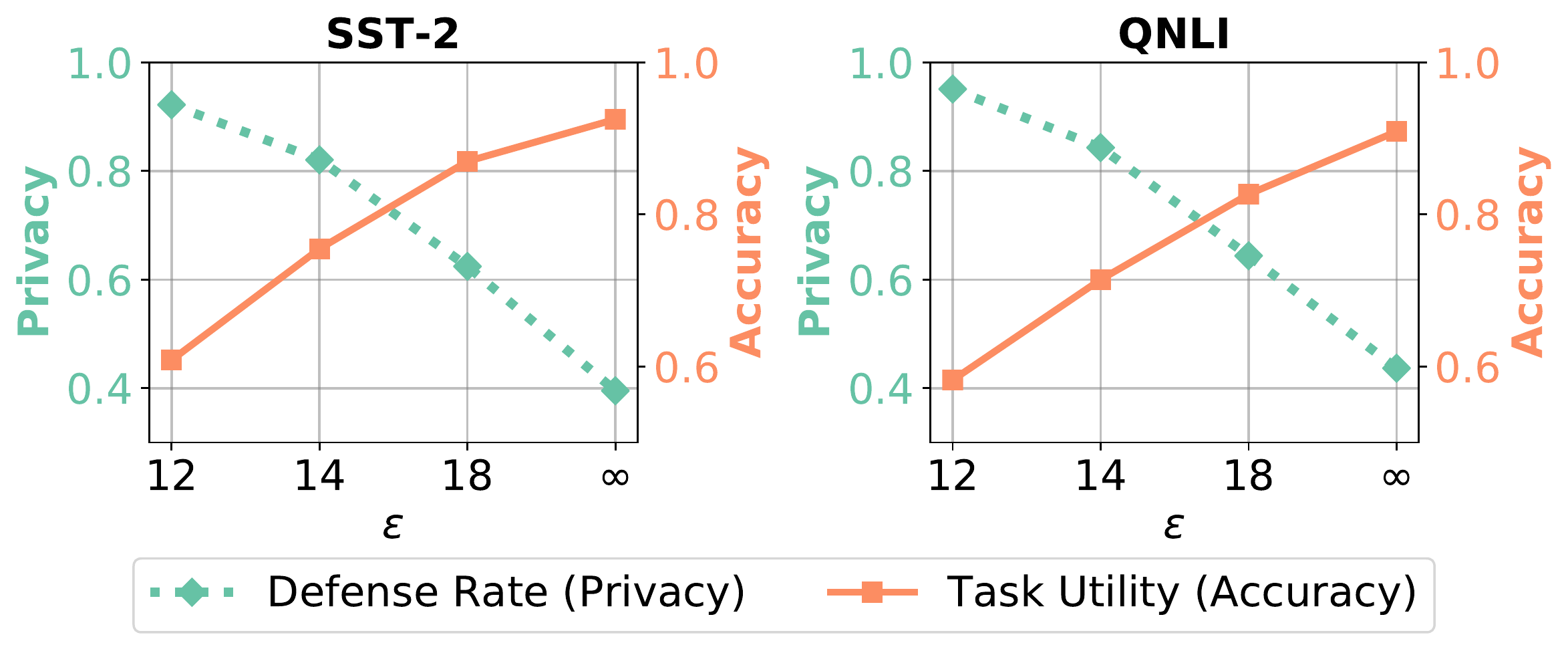}
    \vspace{-20pt}
    %\caption{Results of Mask Token Inference Attack based on \baseDP: We use Defense Rate (of the attack) to measure the privacy and Accuracy to measure the task utility. $\epsilon=\infty$ means ``unsantized.''}
    \caption{Privacy and Utility Tradeoffs of \baseDP
    in terms of Defense Rate (of the Mask Token Inference Attack) 
    versus Accuracy ($\epsilon=\infty$ means ``unsanitized.'')}
    \vspace{-2pt}
    \label{fig:token_inference_attack}
\end{figure}

%!TEX root = main.tex

\begin{table}[!th]
\resizebox{\linewidth}{!}{
\begin{tabular}{l|c|c|c|c}
\hline
\multirow{2}{*}{\textbf{Datasets}} & \multicolumn{2}{c|}{GloVe embeddings} & \multicolumn{2}{c}{BERT embeddings} \\
\cline{2-5}
 & $\V$ & $\V_S$ & $\V$ & $\V_S$ \\
\hline
\hline 
\textbf{SST-2} & 14,730 & 13,258 & \multirow{3}{*}{30,522} & \multirow{3}{*}{27,469} \\
\cline{1-3}
\textbf{MedSTS} & 3,320 & 2,989 & &  \\
\cline{1-3}
\textbf{QNLI} & 88,159 & 79,343 & &  \\
\hline
\end{tabular}
}
\caption{Sizes of vocabularies ($w=0.9$ for $\V_S$)}
\label{tbl:vocabSize}
\end{table}

\begin{figure*}[t!]
% \rotatebox[origin=c]{90}{\bfseries \tiny BERT Embeddings\strut}
    \centering
     \begin{subfigure}[b]{0.95\linewidth}
     \centering
    \includegraphics[width=0.25\linewidth]{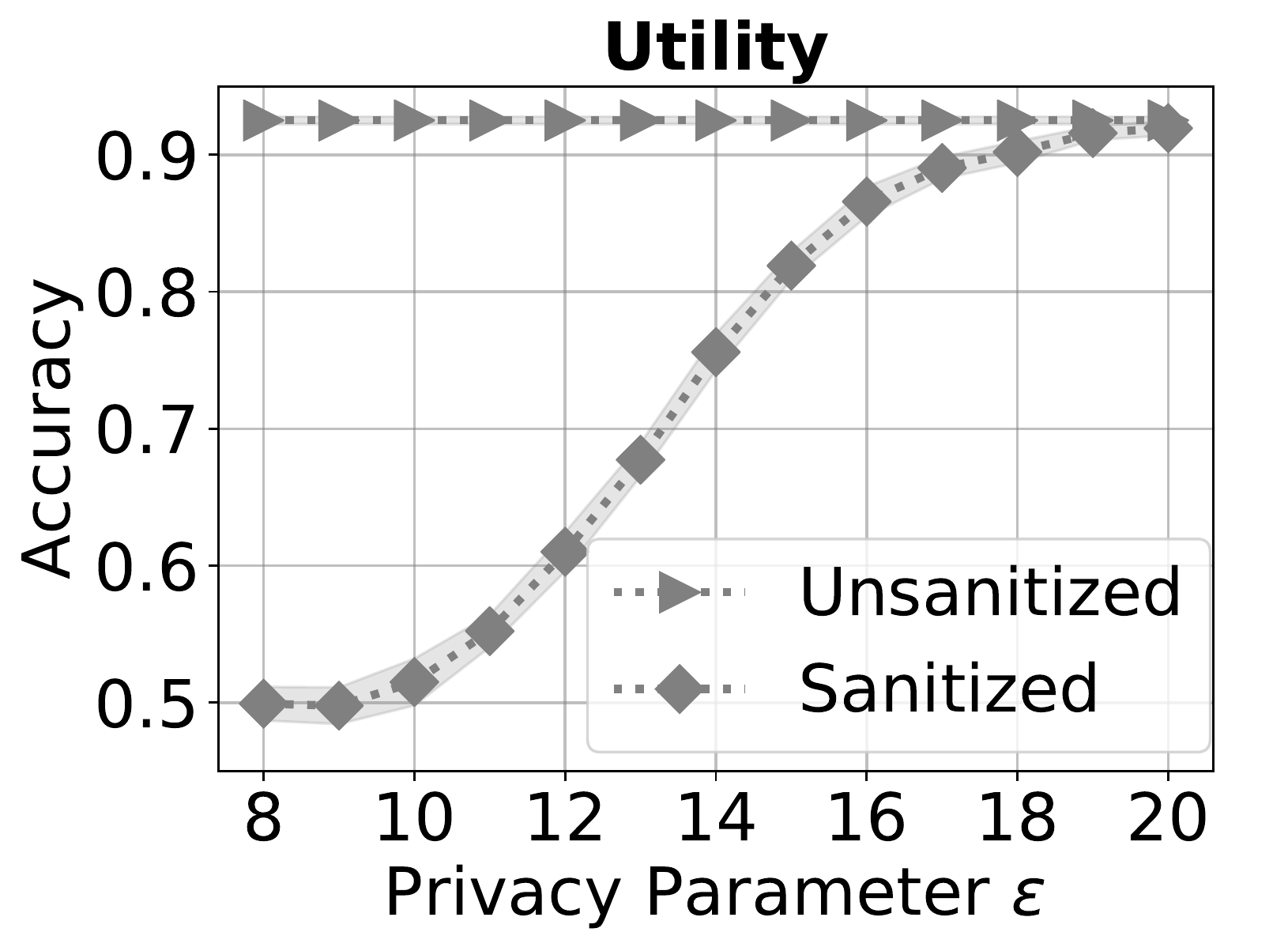}
    \hfill
    % \tikz{\draw[-,gray, densely dashed, thick](1,-0.95) -- (1,1.95);}
    \includegraphics[width=0.73\linewidth]{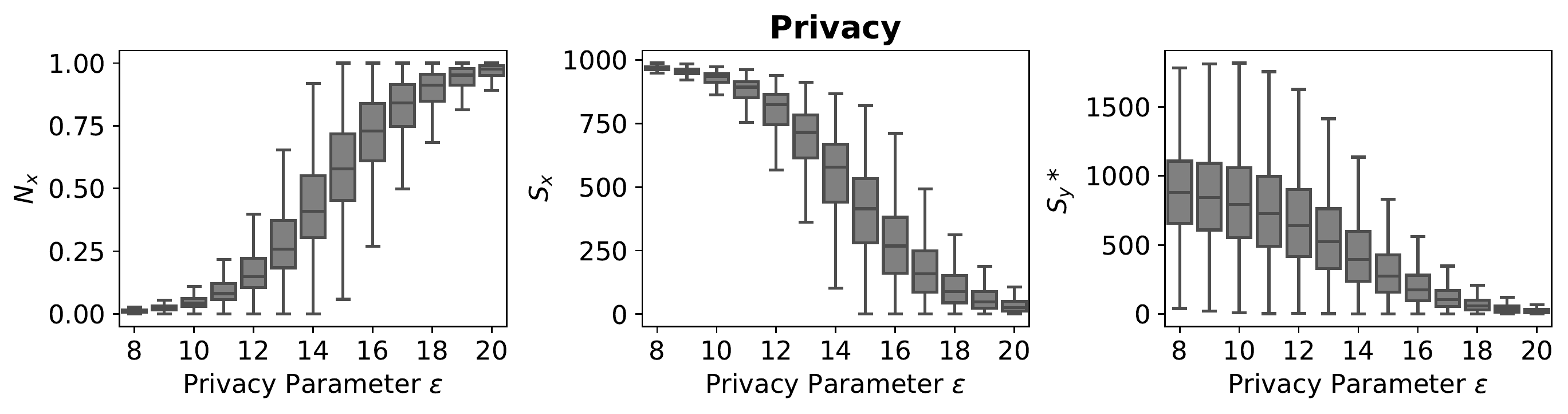} 
 \end{subfigure}
     \tikz{\draw[-,gray, densely dashed, thick](0,0) -- (16,0);}
   \begin{subfigure}[b]{0.95\linewidth}
   \centering
    \includegraphics[width=0.24\linewidth]{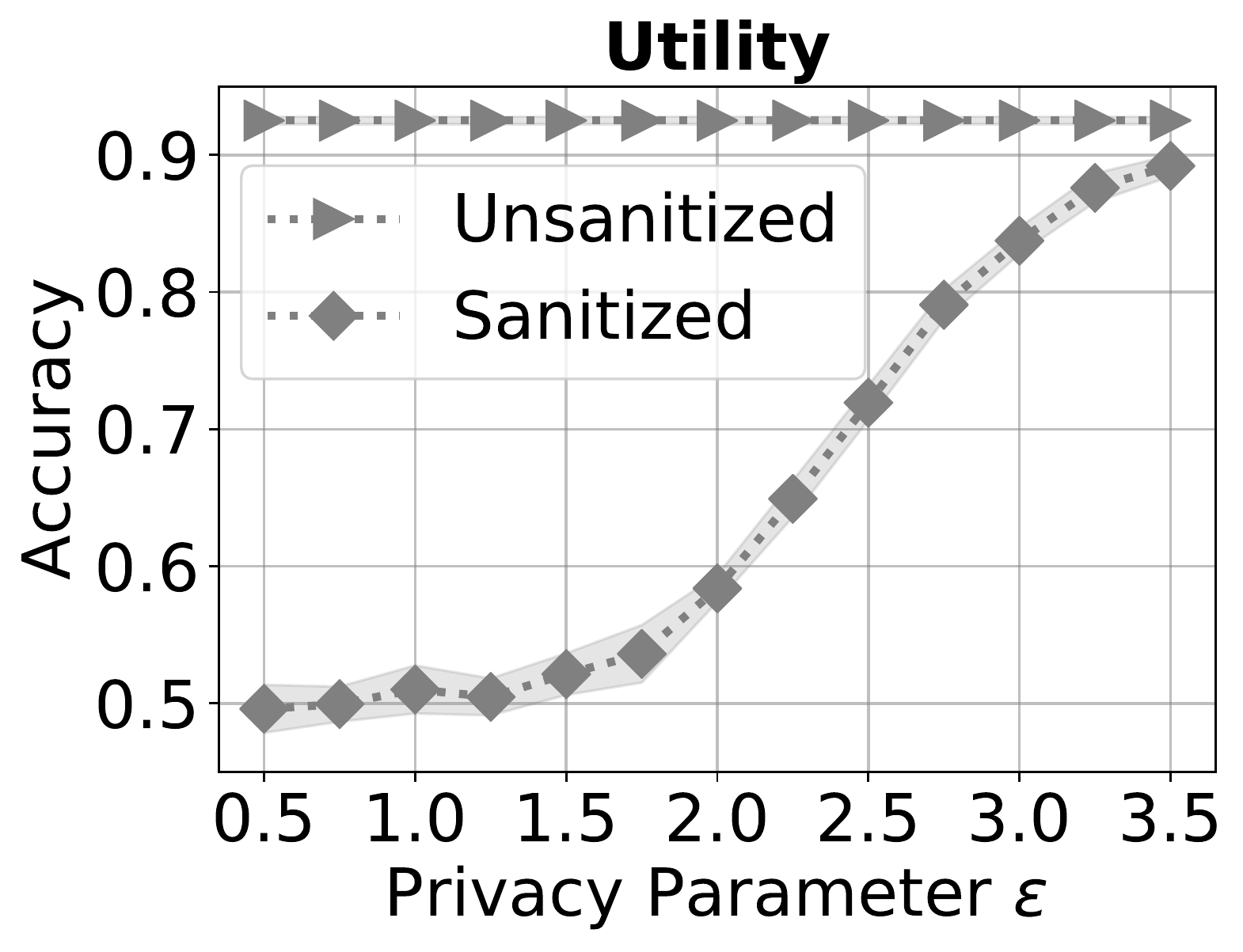}
    \hfill
    \includegraphics[width=0.73\linewidth]{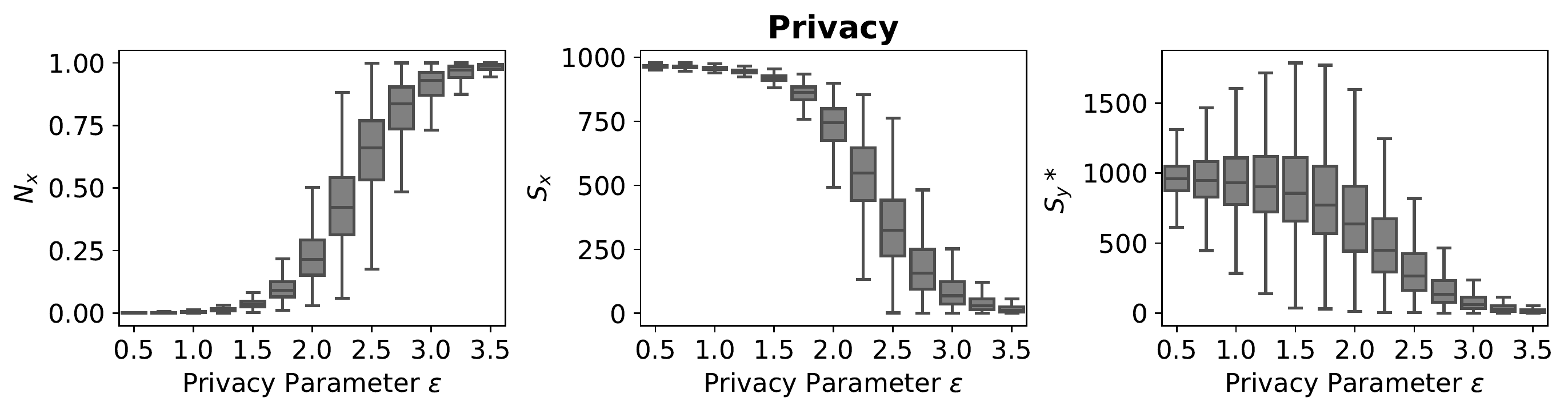}
    \end{subfigure}
    \vspace{-5pt}
    \caption{Influence of privacy parameter $\epsilon$ of \baseDP on the utility and privacy ($N_x$, $S_x$, $S^*_y$) based on the SST-2 dataset: The top panel is based on BERT embeddings, and the bottom panel is based on GloVe embeddings.}
    \label{fig:privacy_stat}
    \vspace{-10pt}
\end{figure*}

\subsection{Mask Token Inference Attack}
\label{sect:mask_token_inference}
%In the face of privacy risks associated with pretrained LMs~\cite{corr/abs-2012-07805}, we empirically show that our sanitization mechanism can actually alleviate this issue by demonstrating a \emph{mask token inference attack}, which we adopt from pretrained BERT models. 

From now on, we adopt the BERT embedding for its superiority.
%(\eg, context-dependent representations, ``out-of-vocab'' support).
%Note that different embeddings require different ``matching'' $\epsilon$ (see Section~\ref{sect:epsilon}).
As (U)MLDP is distance-metric dependent, 
%the noise depends on both $\epsilon$ and $d$.
we need to use different $\epsilon$'s 
(\eg, Figure~\ref{fig:privacy_stat}) 
%keeping the same $\epsilon$ for different $d$ among different embeddings will lead to inconsistency.
to ensure a similar privacy level, specifically, $\epsilon\cdot d$.
%across different embeddings.

Our sanitization mechanisms provide broad protection for seen/unseen attacks at a fundamental level (by sampling noise to directly replace original tokens) with formally-proven DP, \eg, two guesses of the original token with different styles are nearly probable in an attempt of authorship attribution~\cite{sigir/WeggenmannK18}
or other ``indirect'' attacks.
Here, we consider a \emph{mask token inference attack} as a representative study to ``confirm the theory'' by empirically measuring the ``concrete'' privacy level of sanitized texts.

To infer or recover original tokens given the sanitized text, one can let a pretrained BERT model infer the \texttt{MASK} token given its contexts. After all, BERT models are trained via masked language modeling.
For each sanitized text of the downstream (private) corpus, we replace each token sequentially by the special token \texttt{[MASK]} and input the masked text to the pretrained BERT model to obtain the prediction of the \texttt{[MASK]} position. 
Then, we compare the predicted token to the original token in the raw text. 
Figure~\ref{fig:token_inference_attack} reports the defense rate (the proportion of unmatched tokens to total tokens) and task utility of sanitized texts (by \baseDP) as well as unsanitized texts on SST-2 and QNLI. 
We see a privacy-utility trade-off: the more restrictive the privacy guarantee (smaller $\epsilon$), the lower the utility score. 
Notably, we improve the defense rate substantially with only a small amount of privacy loss (\eg, when $\epsilon=16$, \baseDP improves the defense rate by $20\%$ with only $4\%$ task utility loss over the SST-2 dataset in Figure~\ref{fig:token_inference_attack}). 

%!TEX root = main.tex

\begin{table}[t]
\resizebox{\linewidth}{!}{%
\begin{tabular}{l|c|c|c|c}
\hline
\multirow{2}{*}{\textbf{Datasets}} & \multirow{2}{*}{\textbf{$\epsilon$}} & \multicolumn{2}{c|}{\textbf{Utility}} & \multirow{2}{*}{$\Delta_{\privacy}$} \\ \cline{3-4} 
 &  & Original & \textbf{+Pretrain} &  \\ \hline
\multirow{3}{*}{\textbf{SST-2}} & 12 &0.6084  &\textbf{0.6208} &0.0089   \\\cline{2-5} 
 & 14 & 0.7548  &\textbf{0.7731}  & 0.0101   \\ \cline{2-5} 
 & 16 & 0.8698   &\textbf{0.8830}   & -{0.0046}   \\ \hline\hline
\multirow{3}{*}{\textbf{QNLI}} & 12 & 0.5822   & \textbf{0.6037}  &0.0076   \\ \cline{2-5} 
 & 14 &0.7143  &\textbf{0.7309} &-0.0047   \\ \cline{2-5}
 & 16 &0.8265 & \textbf{0.8369} &-0.0039    \\ \hline
\end{tabular}
}
\vspace{-5pt}
\caption{Sanitization-aware pretraining via \baseDP}
\label{tbl:comparisonModel}
\end{table}

\subsection{Effectiveness of Pretraining}
\label{sect:exp_pretrain}
We then show how the sanitization-aware pretraining further improves the utility but does not hurt the original privacy. Specifically, Table~\ref{tbl:comparisonModel} compares the accuracy of sanitization-aware fine-tuning based on the publicly-available \texttt{bert-base-uncased} model and our sanitization-aware pretrained one at different privacy levels on SST-2 and QNLI. Our sanitization-aware pretrained BERT models can obtain a $2\%$ absolute gain on average. 
We conjecture that it can be improved since our pretraining only uses $1/6$ of the data used in the original BERT pretraining and $1$ training epoch as an illustration.

To demonstrate that such utility improvement is not obtained by sacrificing privacy, we record the change of defense rate ($\Delta_{\privacy}$) in launching mask token inference attacks on the original BERT models and our sanitization-aware pretrained BERT models.
%as in Section~\ref{sect:mask_token_inference}. 
As Table~\ref{tbl:comparisonModel} confirmed, the privacy level of our sanitization-aware pretrained model is nearly the same as the original (sometimes even better).
%which empirically proves that the utility improvement does not come at the cost of privacy.

\subsection{Influence of Privacy Parameter $\epsilon$}
\label{sect:epsilon}
We aim at striking a nice balance between privacy and utility by tuning $\epsilon$. 
%In general, a lower $\epsilon$ means a stronger privacy guarantee but leads to lower utility performance. 
%
To empirically show the influence of $\epsilon$, we report the utility and privacy scores over the SST-2 dataset based on \baseDP. The utility score is the accuracy over the test set.
%is the probability of still sampling $x$
We define three metrics to ``quantify'' privacy. 
Firstly, $N_x = \Pr[\M(x)=x]$, which we estimate by the frequency of seeing no replacement by $\M()$.
The output distribution of $x$ has full support over $\V$, \ie, $\Pr[\M(x) = y]>0$ for any $y \in \V$.
Yet, we are interested in the effective support $\S$, a set of $y$'s with cumulative probability larger than a threshold, and then define $S_x$ as its size.
$S_x$ can be estimated by the number of distinct tokens mapped from $x$. 
Both $N_x$ and $S_x$ can be related to two extremes of the R{\'e}nyi entropy~\citep{renyi1961measures}, defined as $H_\alpha(\M(x))= \frac{1}{1-\alpha} \log (\sum_{y \in \V}\Pr[\M(x) = y]^\alpha)$,
% \begin{align*}
%     H_\alpha(\M(x))= \frac{1}{1-\alpha} \log (\sum_{y \in \V}\Pr[\M(x) = y]^\alpha), 
% \end{align*}
with an order $\alpha \geq 0$ and $\alpha \neq 1$.
The two extremes are obtained by setting $\alpha=0$ and $= \infty$, resulting in the Hartley entropy $H_0$ and the min-entropy $H_\infty$. 
This implies that we can also approximate $H_0$ and $H_\infty$ by $\log S_x$ and $-\log N_x$, respectively.
Making them large increases the entropy of the distribution. 

Another important notion is \emph{plausible deniability}~\citep{pvldb/BindschaedlerSG17}, 
%which measures the indistinguishability of inputs. It ensures that 
\ie, a set of $x$'s could have led to an output $y$ with a similar probability. We define $S^*_y$ as the set size,
estimated by the number of distinct tokens mapped to $y$.
%In summary, the larger the $\epsilon$, the higher the accuracy and the lower t

We run \baseDP $1,000$ times for 
%all the tokens in the 
the whole SST-2 dataset vocabulary.
\iffalse
and record: how many times $x$ itself is output by $\M(x)$ (which we can further derive $N_x$ by dividing $1,000$), how many distinct tokens are output by $\M(x)$ (\ie, $S_x$), and how many distinct tokens are sanitized to $x$ (\ie, $S^*_y$). 
All the values w.r.t. $N_x, S_x, S^*_y$ for all the tokens in the vocabulary form three distributions.
\fi
As Figure~\ref{fig:privacy_stat} shows, when $\epsilon$ increases, the utility boosts and $N_x$ increases while $S_x, S^*_y$, and the privacy level of the mechanism decrease, which gives some intuition in picking~$\epsilon$, \eg, for ${\sim}40\%$ probability of replacing each token to a different one based on the BERT embeddings (top panel), we could set $\epsilon=15$ since the median of $N_x$ 
%at $\epsilon=15$ 
is ${\sim}60\%$ and the accuracy is ${\sim}81\%$.

%\noindent\textbf{Influence of different embeddings}. 
%As also can be seen in Figure~\ref{fig:privacy_stat}, different embeddings lead to very different $\epsilon$ setups in general.

\section{Conclusion}
Great predictive power comes with great privacy risks. 
The success of language models enables inference attacks. 
There are only a few works in differentially private (DP) text sanitization, probably due to its intrinsic difficulty. 
A new approach addressing the (inherent) limitation (\eg, in generality) of existing works is thus needed. 

Theoretically, we formulate a new LDP notion, UMLDP, which considers both sensitivity and similarity. While it is motivated by text analytics,
it remains interesting in its own right.
UMLDP enables our natural sanitization mechanisms without the curse of dimensionality faced by existing works.

Practically, we consider the whole PPNLP pipeline and build in privacy at the root with our sanitization-aware pretraining and fine-tuning. 
With our simple and clear definition of sensitivity, our work already achieved promising performance. 
Future research in sophisticated sensitivity measures will further strengthen our approach.

Surprisingly, our PPNLP solution is discerning like a cryptographic solution: 
it is kind (maintains high utility) to the good but not as helpful to the bad (not boosting up inference attacks). 
We hope our results with different metrics for quantifying privacy can provide more insights in privacy-preserving NLP and make it accessible to a broad audience.

\section*{Acknowledgements}
Authors at OSU are sponsored in part by the PCORI Funding ME-2017C1-6413, the Army Research Office under cooperative agreements W911NF-17-1-0412, NSF Grant IIS1815674, NSF CAREER \#1942980, and Ohio Supercomputer Center \cite{OhioSupercomputerCenter1987}.
The views and conclusions contained herein are those of the authors and should not be interpreted as representing the official policies,
either expressed or implied, of the Army Research Office or the U.S. Government. The U.S. Government is authorized to reproduce and distribute reprints for Government purposes notwithstanding any copyright notice herein.

Sherman Chow's research is supported by General Research Fund (CUHK 14210319) of UGC, HK. 
Authors at CUHK would like to thank Florian Kerschbaum for his inspiring talk given at CUHK that stimulated this research.

\clearpage
\appendix

\section{Supplementary Formalism Details}
\subsection{Definition of ULDP}
%and UMLDP}
\label{sect:uldp}
\begin{definition}[$(\V_S,\V_P,\epsilon)$-ULDP~\citep{uss/Murakami019}]
\label{def:uldp-formal}
Given $(\V_S = \V_P) \subseteq \V$, a privacy parameter $\epsilon \geq 0$, $\M$ satisfies $(\V_S,\V_P,\epsilon)$-ULDP if it satisfies the properties:
\\
i) for any $x,x' \in \mathcal{V}$ and any $y\in \mathcal{V}_P$, we have 
\begin{align*}
\Pr[\mathcal{M}(x)=y] \leq e^\epsilon \Pr[\mathcal{M}(x')=y];
\end{align*}
ii) for any $y \in \mathcal{V}_U$, there is an $x\in \mathcal{V}_N$ such that
\begin{align*}
\Pr[\mathcal{M}(x)=y] > 0; \Pr[\mathcal{M}(x')=y] = 0 \ \text{for}\ x\neq x'.
\end{align*}
\end{definition}

\iffalse
Incorporating the distance metric into the ULDP definition needs a bit of tailor adjustment due to the reversibility, meaning that an unprotected output must be mapped from a non-sensitive input. 
Consider the (semantic) distance between unsanitized $x$ and sanitized $y$ can be evaluated by a function $f(x,y)$ and computing a sampling probability proportional to it via:
$$\Pr[\mathcal{M}(x)=y] = \frac{f(x,y)}{\sum_{y'}f(x,y')}.$$

In MLDP, the output space is the same $\V$ for any pair of $x,x'$, and we have
$$\frac{\Pr[\mathcal{M}(x)=y]}{\Pr[\mathcal{M}(x')=y]} = \frac{f(x,y)\cdot \sum_{y'\in \V}f(x',y')}{f(x',y)\cdot \sum_{y'\in \V}f(x,y')},$$
which can be bounded by $e^{\epsilon \cdot d(x,x')}$.
In contrast, in UMLDP, while the output space of $x \in \V_S$ is $\V_P$, that of 
$x' \in \V_N$ is not only $\V_P$ but possibly also includes an ``invertible'' element, say, $y^*$. The above ratio of probabilities thus becomes
$$\frac{\Pr[\mathcal{M}(x)=y]}{\Pr[\mathcal{M}(x')=y]} = \frac{f(x,y)\cdot \sum_{y'\in \V_P \cup \{y^*\}}f(x',y')}{f(x',y)\cdot \sum_{y'\in \V_P}f(x,y')} $$ 
$$ \leq e^{\epsilon \cdot d(x,x')} + \frac{f(x,y)\cdot f(x',y^*)}{f(x',y)\cdot \sum_{y'\in \V_P}f(x,y')}.$$ 
The additional term could still be bounded by $e^{\epsilon \cdot d(x,x')}$, but it would be too loose, which doubles~$\epsilon$. 
Instead, we loosen the original bound slightly by introducing an additive parameter $\epsilon_0$.
\fi

\subsection{Differential Privacy Guarantee}
\label{sect:proof}
%Below shows the proof for Theorem~\ref{thm:base}.

\begin{proof}[Proof of Theorem~\ref{thm:base}]
Consider $L = 1$, \ie, $D = \langle x \rangle$. 
For another document $D'$ with $x' \in \V \setminus \{x\}$ and a possible output $y \in \V$:
%, we have
\begin{align*}
   &\frac{\Pr[\M(x)=y]}{\Pr[\M(x')=y]}\\
   = & \frac{C_{x} \cdot e^{-\frac{1}{2} \epsilon \cdot d_{\euc}(\phi(x),\phi(y))}}{C_{x'} \cdot e^{-\frac{1}{2} \epsilon \cdot d_{\euc}(\phi(x'),\phi(y))}} \\
   = & \frac{C_x}{C_{x'}}\cdot e^{\frac{1}{2}\epsilon \cdot [d(x',y)-d(x,y)]} \\
   \leq & \frac{C_x}{C_{x'}}\cdot e^{\frac{1}{2} \epsilon \cdot d(x,x')} \\
   = & \frac{\sum_{y' \in \V}e^{-\frac{1}{2} \epsilon \cdot d(x',y')}}{\sum_{y' \in \V}e^{-\frac{1}{2} \epsilon \cdot d(x,y')}} \cdot e^{\frac{1}{2} \epsilon \cdot d(x,x')} \\
   \leq & e^{\frac{1}{2} \epsilon \cdot d(x,x')} \cdot e^{\frac{1}{2} \epsilon \cdot d(x,x')} = e^{\epsilon \cdot d(x,x')}
\end{align*}
The proof, showing \baseDP ensures $\epsilon \cdot d(x,x')$-LDP, mainly relies on the triangle inequality of $d$.
To generalize to the case of $L > 1$, we sanitize every token $x_i$ in $D$ independently, and thus:
%we have 
$$\Pr[\M(D) = \hat{D}] = \prod^L_{i=1} \Pr[\M(x_i)=y_i].$$
Then, for any $D,D'$, the privacy bound is given as 
$$\frac{\Pr[\M(D) = \hat{D}]}{\Pr[\M(D') = \hat{D}]} \leq e^{\epsilon \cdot \sum^L_{i=1}d(x_i,x'_i)},$$
which follows from the composability.
\end{proof}

%The proof for Theorem~\ref{thm:enhanced} is given below.

\begin{proof}[Proof of Theorem~\ref{thm:enhanced}]
Consider the case $L=1$ with $D = x$ and $D' = x'$. 
For $x,x' \in \V_S$, the output $y$ is restricted to $\V_P$, with the proof identical to the above theorem (as \baseDP is run over $\V_S, \V_P$).

For $x,x' \in \V_N$ and $y \in \V_P$, we have 
\begin{align*}
     \frac{\Pr[\M(x)=y]}{\Pr[\M(x')=y]} & = \frac{p \cdot C_{x} \cdot e^{-\frac{1}{2} \epsilon \cdot d_{\euc}(\phi(x),\phi(y))}}{p \cdot C_{x'} \cdot e^{-\frac{1}{2} \epsilon \cdot d_{\euc}(\phi(x'),\phi(y))}} \\
     & \leq e^{\epsilon \cdot d(x,x')}.
\end{align*}
For $x \in \V_S$, $x' \in \V_N$, and $y \in \V_P$, we have
\begin{align*}
     \frac{\Pr[\M(x)=y]}{\Pr[\M(x')=y]} & = \frac{C_{x} \cdot e^{-\frac{1}{2} \epsilon \cdot d_{\euc}(\phi(x),\phi(y))}}{p \cdot C_{x'} \cdot e^{-\frac{1}{2} \epsilon \cdot d_{\euc}(\phi(x'),\phi(y))}} \\
     & \leq \frac{1}{p} \cdot e^{\epsilon \cdot d(x,x')} = e^{\epsilon \cdot d(x,x') + \epsilon_0}.
\end{align*}

\iffalse
This shows that \baseDPp ensures the property i) of UMLDP.
The property ii) also holds due to the probability $(1-p)$ for $x \in \V_N$. 
Recall that t
\fi
The probability for $x \in \V_N$ is $(1-p)$. 
The above inequalities thus show that \baseDPp ensures the properties of UMLDP.
Similarly, we use the composability to generalize for $L >1$.
\end{proof}

\iffalse
\subsubsection{Discussion on $p$.}
If we do not assign a fixed probability $(1-p)$ of mapping $x \in \V_N$ to itself, \ie, still using the Eq.~(\ref{eq:baseline}) which takes in $d(x,x)=0$.
Then for $x \in \V_S$, $x' \in \V_N$, and $y \in \V_P$, we have 
\begin{align*}
     \frac{\Pr[\M(x)=y]}{\Pr[\M(x')=y]} & = \frac{C_{x} \cdot e^{-\frac{1}{2} \epsilon \cdot d_{\euc}(\phi(x),\phi(y))}}{C_{x'} \cdot e^{-\frac{1}{2} \epsilon \cdot d_{\euc}(\phi(x'),\phi(y))}},
\end{align*}
with $C_{x'} = (\sum_{y' \in \V_P \cup \{x'\}}e^{-\frac{1}{2} \epsilon \cdot d_{\euc}(\phi(x'),\phi(y'))})^{-1}$.
The latter $e$ terms are bounded by $e^{\frac{1}{2}\epsilon \cdot d(x,x')}$.
Now let us consider 
\begin{align*}
\frac{C_x}{C_{x'}} & = \frac{\sum_{y' \in \V_P \cup \{x'\}}e^{-\frac{1}{2} \epsilon \cdot d_{\euc}(\phi(x'),\phi(y'))}}{\sum_{y' \in \V_P }e^{-\frac{1}{2} \epsilon \cdot d_{\euc}(\phi(x),\phi(y'))}} \\
& \leq e^{\frac{1}{2}\epsilon \cdot d(x,x')} + \frac{1}{\sum_{y' \in \V_P }e^{-\frac{1}{2} \epsilon \cdot d_{\euc}(\phi(x),\phi(y'))}}.
\end{align*}
Combining them together, we obtain
\begin{align*}
     \frac{\Pr[\M(x)]}{\Pr[\M(x')]} & \leq e^{\epsilon \cdot d(x,x')} + \frac{e^{\frac{1}{2}\epsilon \cdot d(x,x')}}{\sum_{y' \in \V_P }e^{-\frac{1}{2} \epsilon \cdot d_{\euc}(\phi(x),\phi(y'))}}
\end{align*}
It is unclear how to give an upper bound of the latter additive term as it relies on the choice of $x$ yielding the minimum value of the denominator.
\fi

\begin{table*}[t]
\resizebox{\linewidth}{!}{%
\begin{tabular}{l|c|l}
\hline
\multicolumn{3}{l}{\textbf{Dataset: SST-2}} \\ \hline
Mechanisms & $\epsilon$ & \begin{tabular}[c]{@{}l@{}}\textbf{Original Text:} \\ it   's a charming and often affecting journey . \end{tabular}\\ \hline
\multirow{3}{*}{\baseDP} & 1 & heated collide. charming   activity cause challenges beneath tends \\ \cline{2-3} 
 & 2 & worse beg, charming   things working noticed journey basically \\ \cline{2-3} 
 & 3 & all 's. charming and   often already journey demonstrating \\ \hline
\multirow{3}{*}{\baseDPp} & 1 & it unclear a charming and   often hounds journey \\ \cline{2-3} 
 & 2 & it exaggeration a   charming feelings often lags journey . \\ \cline{2-3} 
 & 3 & it 's a tiniest picked   often affecting journey . \\ \hline\hline
\multicolumn{3}{l}{\textbf{Dataset: QNLI}} \\ \hline
Mechanisms & $\epsilon$ & \begin{tabular}[c]{@{}l@{}}\textbf{Original Text:} \\
When did Tesla move to New York City? \\ In 1882, Tesla began working for the Continental Edison Company in France, \\ designing and making improvements to electrical equipment.\end{tabular} \\ \hline
\multirow{3}{*}{\baseDP} & 1 & \begin{tabular}[c]{@{}l@{}}43 trapper Gaga MCH digest sputtering avenged Forced Laborers \\ Homage Ababa afer psychic 51,000  intercity lambasting nightmare--confederate Frontier \\ Britian Manor Londres shards pilot Mining faster alone Thessalonica Bessemer Lie Columbus\end{tabular} \\ \cline{2-3} 
 & 2 & \begin{tabular}[c]{@{}l@{}}blame least ethos did tenth ballot Condemnation critical filmed \\ In 1883 3200 Conversion pushing   7:57 enabling Town stamp Time downwards Peterson France, \\ GSA emulating   addresses appealing 47.4 electrical pull refreshing\end{tabular} \\ \cline{2-3} 
 & 3 & \begin{tabular}[c]{@{}l@{}}Wave did Tesla It way Dru Tully breaking? \\ Tupelo 1875, Tesla began escaped for announcing Continental   Edison Company in France \\ However designing and making improvements to   electrical Chongqing add\end{tabular} \\ \hline
\multirow{3}{*}{\baseDPp} & 1 & \begin{tabular}[c]{@{}l@{}}Rodgers did Sung move to New plantation City ? \\ In K. innumerable Gunz began working sliding the   Sultans Edison Company structured France\\  beaching designing disseminate   making tribunals to lackluster equipment 40-foot\end{tabular} \\ \cline{2-3} 
 & 2 & \begin{tabular}[c]{@{}l@{}}vaults did Tesla chunks introduces Teknologi Eyes City ?\\ In 866 , Tesla began working for the Analytical Edison   Company Butterfly France , \\ designing Sias siblings Noting circumventing   electrical orient .\end{tabular} \\ \cline{2-3} 
 & 3 & \begin{tabular}[c]{@{}l@{}}When did Tesla guideline to New York City ? \\ In 1885 , Tesla MG working for the Continental Edison   Company in France , \\ translating and dreamed improvements ascertain electrical lookout .\end{tabular} \\ \hline
\end{tabular}
}
\caption{Qualitative examples from the SST-2 and QNLI datasets: Sanitized text by our mechanisms at different privacy levels based on GloVe embeddings}
\label{tbl:qual_ex}
\end{table*}

\subsection{Qualitative Observations}
\label{sect:qual_observation}
Below, we focus on \baseDP sanitizing a single token $x$.
We first make two extreme cases explicit.

\noindent
(1) When $\epsilon = 0$, the distribution in Eq.~(\ref{eq:baseline}) becomes
$\Pr[\M(x)=y] = \frac{1}{|\V|}, \forall y \in \V$.
\baseDP is perfectly private since $y$ is uniformly sampled at random, independent of $x$.
Yet, such a $y$ does not preserve any information of $x$.

\noindent
(2) When $\epsilon \rightarrow \infty$, we have 
$\Pr[\M(x)=x] \gg \Pr[\M(x)=y], y \in \V \setminus \{x\}$.
$\Pr[\M(x)=x]$ dominates others since $d(x,x)=0$
and $d(x,y) > 0$.
This loses no utility as $x$ almost stays unchanged, yet provides no privacy either.

For a general $\epsilon \in (0, \infty)$, the distribution has full support over $\V$, \ie, we have a non-zero probability for any possible $y \in \V$ such that $\M(x) = y$.
Also, given $y,y' \in \V$ with $d(x,y) < d(x,y')$, we have $\Pr[\M(x)=y] > \Pr[\M(x)=y']$.
As $\epsilon$ increases, $\Pr[\M(x)=y]$ for the $y$'s with large $d(x,y)$ goes smaller (and even approaches 0).
This means that the output distribution becomes ``skewed,'' \ie, the outputs concentrate on those $y$'s with small $d(x,y)$.
This is good for utility, which stems from the semantics
preservation of every 
%single 
token. 
On the contrary, too much concentration weakens the privacy.

For \baseDPp, the above results directly apply to the case $x \in V_S$ (as \baseDP is run over $\V_S$ and $\V_P$).
There is an extra $p$ determining whether a $x \in \V_N$ is mapped to a $y \in \V_P$.
If so, the results are similar except with an extra multiplicative $p$.
A larger $p$ leads to stronger privacy as the probability $(1-p)$ of $x$ being unchanged becomes smaller.

\section{Qualitative Examples}
\label{sect:qual_ex}
Table~\ref{tbl:qual_ex} shows two examples of sanitized texts output by \baseDP and \baseDPp at different privacy levels from the SST-2 and QNLI datasets.

% \newpage
% \section{Supplementary Experiments}
% \label{sect:supp_exp}
% \subsection{Influence of privacy parameter $\epsilon$}
% In our main text, we discuss the influence of privacy parameter $\epsilon$ based on \baseDP with BERT embeddings. Here, we show the same setting results based on GLoVe embeddings in Figure \ref{apdx_fig:privacy_stat}.

% \subsection{Influence of hyperparameters $w$ and $p$ in \baseDPp}

% In \baseDPp, we have two hyperparameters: the first $w\%$ low-frequency token (according to a certain corpus) as sensitive tokens and the biased coin probability $p$ for sanitizing non-sensitive tokens. 

%\paragraph{Earlier Works in Text Sanitization.}
%\add{SynTF~\cite{sigir/WeggenmannK18} assumes only term-frequency or tf-idf (inverse document frequency) vectors are available. 
%It samples words from a tf-idf based distribution for replacement.
%Its main drawback lies in the fact that the non-ordered sanitized texts it produces only maintain content-level features (specifically, term-frequency vectors).
%}

%\cite{icdm/FeyisetanDD19}
\section{Supplementary Related Works}\label{sect:appendix-related}
Privacy is a practically relevant topic that also poses research challenges of diverse flavors. Below, we discuss some ``less-directly'' relevant works, showcasing some latest advances in AI privacy.
\iffalse
%~\cite{manuscript/SoK21}
According to our study, a non-exhaustive list of other works from the cryptography, security, and privacy communities quickly reaches $100$ papers.
\fi

\paragraph{Cryptographic Protection of (Text) Analytics.}
There has been a flurry of results improving privacy-preserving machine-learning frameworks (\eg, \citep{nips/LouFF020}), which make use of cryptographic tools such as homomorphic encryption and secure multi-party computation (SMC) for general machine/deep learning.
These cryptographic designs can be adapted for many NLP tasks in principle.
Nevertheless, they will slow down computations by orders of magnitude since cryptographic tools, especially fully homomorphic encryption, are generally more heavyweight than the DP approaches.
One might be tempted to replace cryptography with \emph{ad hoc} heuristics. Unfortunately, it is known to be error-prone (\eg, a recently proposed attack~\cite{ijcai/WongMWNC20} can recover model parameters during ``oblivious'' inference).

A recent trend (\eg,~\cite{popets/WaghTBKMR21}) relies on multiple non-colluding servers to perform SMC for secure training.
However, SMC needs multiple rounds of communication. It is thus more desirable to have a dedicated 
%network 
connection among the servers.

Albeit with better utility (than DP-based designs), cryptographic approaches mostly consider immunity against membership inference~\cite{sp/ShokriSSS17} to be out of their protection scope since DP mechanisms could be applied over the training data before the cryptographic processing. 

There is a growing interest in privacy-preserving analytics in the NLP community too. Very recently, TextHide~\citep{emnlp/HuangSCLA20} devises an ``encryption'' layer for the hidden representations. Unfortunately, it is shown to be insecure by cryptographers and privacy researchers~\citet{corr/abs-2011-05315}.

\paragraph{Hardware-Aided Approaches.}
%Hardware-aided approaches are getting more popular in privacy-preserving machine-learning computation. For example, 
GPU can compute linear operations in a batch much faster than CPU. Nevertheless, we still need a protection mechanism in using GPU, another protection mechanism for the non-linear operations, and their secure integration. In general, utilizing GPU for privacy-preserving machine-learning computations is non-trivial (\eg, see~\cite{uss/NgC21} for an extended discussion).

To exploit the parallelism of GPU while minimizing the use of cryptography, 
one can resort to a trusted processor (\eg, Intel SGX) for performing non-linear operations within its trusted execution environment (TEE)
Note that one still needs to use cryptographic protocols to outsource the linear computation to (untrusted) GPU. Slalom~\cite{iclr/TramerB19} is such a solution that supports privacy-preserving inference. Training is a more challenging task that was left as an open challenge. Recently, it is solved by Goten~\cite{aaai/goten21}. Notably, both works are from cryptographers but also get recognized by the AI community.

Finally, we remark that the use of TEE is not a must in GPU-enabled solutions. 
For example, GForce~\cite{uss/NgC21} is one of the pioneering works that proposes GPU-friendly protocols for non-linear layers with other contributions.

\end{document}